\documentclass[a4,a4wide]{article}
\usepackage{aaai}
\usepackage{times}
\usepackage{helvet}
\usepackage{courier}
\usepackage{txfonts} 

\usepackage[french]{babel}
\usepackage[utf8]{inputenc}
\usepackage{listings}
\usepackage{lscape}
\usepackage{url}
\usepackage{enumerate}
\usepackage{amssymb}
\usepackage{tikz}
\usepackage{xcolor}
\usepackage[ruled,vlined]{algorithm2e}
\usepackage{array}

\usetikzlibrary{shapes}
\usetikzlibrary{arrows}

\newtheorem{theorem}{Theorem}

\newtheorem{proposition}{Proposition}

\newtheorem{definition}{Definition}
\newtheorem{example}{Example}

\newcommand{\fun}[1]{\ensuremath\mbox{\sl #1}}

\newenvironment{proof}{
	\noindent {\em Proof:}
}{
	\hfill$\square$
}

\newcounter{problem}

\newcolumntype{L}[1]{>{\raggedright\let\newline\\\arraybackslash\hspace{0pt}}m{#1}}
\newcolumntype{C}[1]{>{\centering\let\newline\\\arraybackslash\hspace{0pt}}m{#1}}
\newcolumntype{R}[1]{>{\raggedleft\let\newline\\\arraybackslash\hspace{0pt}}m{#1}}

\newcommand{\NP}[0]{NP}
\newcommand{\NPc}[0]{\NP-complete}
\newcommand{\coNP}[0]{co\NP}
\newcommand{\coNPc}[0]{co\NPc}

\newcommand{\R}[0]{\mathcal R}

\newcommand{\unifier}[0]{\mu}
\newcommand{\unified}[0]{\diamond}

\newcommand{\gerd}[1]{#1^{U}}
\newcommand{\grd}[1]{#1^{D}}
\newcommand{\gerdc}[1]{#1^{U+}}

\newcommand{\wa}[0]{wa}

\newcommand{\fd}[0]{fd}

\newcommand{\ar}[0]{ar}

\newcommand{\agrd}[0]{a$-$grd}

\newcommand{\ja}[0]{ja}
\newcommand{\swa}[0]{swa}
\newcommand{\msa}[0]{msa}
\newcommand{\mfa}[0]{mfa}

\newcommand{\pg}[0]{PG}
\newcommand{\ppg}[1]{\pg^{#1}}
\newcommand{\xpg}[0]{\ppg{X}}
\newcommand{\fpg}[0]{\ppg{F}}
\newcommand{\dpg}[0]{\ppg{D}}
\newcommand{\upg}[0]{\ppg{U}}
\newcommand{\PG}[1]{\pg(#1)}

\newcommand{\XPG}[1]{\xpg(#1)}
\newcommand{\FPG}[1]{\fpg(#1)}
\newcommand{\DPG}[1]{\dpg(#1)}
\newcommand{\UPG}[1]{\upg(#1)}

\newcommand{\vxpg}[4]{ $\guilsinglleft$#4,#3$\guilsinglright$ }
\newcommand{\vpos}[2]{\ensuremath{\vxpg{}{}{#2}{#1}}}

\renewcommand{\t}[0]{\vec t}

\newcommand{\X}[0]{\vec X}
\newcommand{\Y}[0]{\vec Y}
\newcommand{\Z}[0]{\vec Z}

\newcounter{gerdRuleID}
\newcounter{gerdTermID}
\newcounter{gerdUnifID}
\newcounter{gerdNbBodyTerms}
\newcounter{gerdNbHeadTerms}

\newenvironment{gerd:figure}{
	\setcounter{gerdRuleID}{0}
	\setcounter{gerdUnifID}{0}
	\begin{tikzpicture}
}{
	\end{tikzpicture}
}

\tikzstyle{gerd:node}=[
]

\tikzstyle{gerd:fake}=[
	gerd:node,
	inner sep=0mm
]

\tikzstyle{gerd:term}=[
	gerd:node,
	circle,
	minimum size=1mm,
	inner sep=0mm,
	draw=none
]

\tikzstyle{gerd:edge}=[
	->,
	draw=black
]

\tikzstyle{gerd:edge-norm}=[
	gerd:edge,
]

\tikzstyle{gerd:edge-spec}=[
	gerd:edge-norm
]

\tikzstyle{gerd:edge-unif}=[
	gerd:edge,
]

\tikzstyle{gerd:border}=[
	-,
	draw=black
]

\tikzstyle{gerd:border-ext}=[
	gerd:border,
	thick
]

\tikzstyle{gerd:border-int}=[
	gerd:border,
	draw=black!50!white,
	dashed
]

\newcounter{gerdTmpCounter}
\newcounter{gerdMaxTermCounter}
\newcommand{\gerdRule}[9][0]{
	\refstepcounter{gerdRuleID}

	\setcounter{gerdNbBodyTerms}{0}
	\setcounter{gerdNbHeadTerms}{0}
	\foreach \tb in {#6} { \refstepcounter{gerdNbBodyTerms} }
	\foreach \th in {#7} { \refstepcounter{gerdNbHeadTerms} }

	\foreach \i [evaluate=\i as \x using (((#4-#2)/4)*\i)+#2] in {0, 1, 2, 3, 4} {
		\node[gerd:fake] (fakeT\thegerdRuleID_\i) at (\x,#3) {};
	}
	\foreach \i [evaluate=\i as \x using (((#4-#2)/4)*\i)+#2, evaluate=\i as \y using #3+#3-#5] in {0, 1, 2, 3, 4} {
		\node[gerd:fake] (fakeB\thegerdRuleID_\i) at (\x,\y) {};
	}

	\draw[gerd:border-ext] (fakeT\thegerdRuleID_0.center) -- (fakeT\thegerdRuleID_4.center)
						-- (fakeB\thegerdRuleID_4.center) -- (fakeB\thegerdRuleID_0.center)
						-- (fakeT\thegerdRuleID_0.center);
	\draw[gerd:border-int] (fakeT\thegerdRuleID_2.center) -- (fakeB\thegerdRuleID_2.center);

	\ifthenelse{\equal{#1}{1}}{
		\pgfmathparse{#2+(3*((#4-#2)/4))}
	}{
		\pgfmathparse{#2+((#4-#2)/4)}
	}
	\setcounter{gerdTermID}{1}
	\foreach \pp/\pt/\ppl [evaluate=\pp as \x using \pgfmathresult,
				  	  evaluate=\pp as \y using #3-(((#5-#3)/(\thegerdNbBodyTerms+1))*\thegerdTermID)]
			in {#6} {	
				\setcounter{gerdTmpCounter}{1}
				\setcounter{gerdMaxTermCounter}{0}
				\foreach \ptt in \pt { \refstepcounter{gerdMaxTermCounter} }
		\node[gerd:term]
			(tb\thegerdRuleID_\pp{\pt}\ppl) at (\x,\y) {$\pp(
				\foreach \ptt in \pt {
					\ifthenelse{\equal{\thegerdTmpCounter}{\ppl}}{{\underline{\bf \ptt}}}{\ptt}
					\ifthenelse{\equal{\thegerdTmpCounter}{\thegerdMaxTermCounter}}{}{,}
					\refstepcounter{gerdTmpCounter}
				}
				)$};
		\refstepcounter{gerdTermID}
	}

	\ifthenelse{\equal{#1}{1}}{
		\pgfmathparse{#2+((#4-#2)/4)}
	}{
		\pgfmathparse{#2+(3*((#4-#2)/4))}
	}
	\setcounter{gerdTermID}{1}
	\foreach \pp/\pt/\ppl [evaluate=\pp as \x using \pgfmathresult,
				  evaluate=\pp as \y using #3-(((#5-#3)/(\thegerdNbHeadTerms+1))*\thegerdTermID)]
			in {#7} {	
				\setcounter{gerdTmpCounter}{1}
				\setcounter{gerdMaxTermCounter}{0}
				\foreach \ptt in \pt { \refstepcounter{gerdMaxTermCounter} }
		\node[gerd:term]
			(th\thegerdRuleID_\pp{\pt}\ppl) at (\x,\y) {$\pp(
				\foreach \ptt in \pt {
					\ifthenelse{\equal{\thegerdTmpCounter}{\ppl}}{{\underline{\bf \ptt}}}{\ptt}
					\ifthenelse{\equal{\thegerdTmpCounter}{\thegerdMaxTermCounter}}{}{,}
					\refstepcounter{gerdTmpCounter}
				}
				)$};
		\refstepcounter{gerdTermID}
	}

	\foreach \ppa/\ppb in {#8} {
		\draw[gerd:edge-norm] (tb\thegerdRuleID_\ppa.east) -- (th\thegerdRuleID_\ppb.west);
		\foreach \te in {#9} {
			\draw[gerd:edge-spec] (tb\thegerdRuleID_\ppa.east) -- (th\thegerdRuleID_\te.west);
		}
	}
}

\newcommand{\gerdUnif}[5][]{
	\foreach \th/\tb in {#4} {
		\draw[gerd:edge-unif,#1] (th#2_\th) to node[above]{#5} (tb#3_\tb);
	}
}

\newenvironment{gp:figure}{
	\begin{tikzpicture}
}{
	\end{tikzpicture}
}

\tikzstyle{gp:node}=[
	gerd:node
]

\tikzstyle{gp:fake}=[
	gp:node,
	gerd:fake
]

\tikzstyle{gp:predpos}=[
	gp:node,
	rectangle,
	minimum size=1mm,
	draw=black
]

\tikzstyle{gp:edge}=[
	gerd:edge
]

\tikzstyle{gp:edge-norm}=[
	gp:edge,
	gerd:edge-norm
]

\tikzstyle{gp:edge-spec}=[
	gp:edge,
	gerd:edge-spec
]

\tikzstyle{fake}=[
	inner sep=0mm
]

\tikzstyle{rc:node}=[
]

\tikzstyle{rc:edge}=[
	-,
	draw=black
]

\tikzstyle{rc:complexity-edge}=[
	-,
	thick,
	densely dotted,
	draw
]

\tikzstyle{rc:P-edge}=[
	rc:complexity-edge,
	draw=green!66!black
]

\tikzstyle{rc:NP-edge}=[
	rc:complexity-edge,
	draw=blue!66!black
]

\tikzstyle{rc:ET-edge}=[
	rc:complexity-edge,
	draw=orange!66!black
]

\tikzstyle{rc:ETT-edge}=[
	rc:complexity-edge,
	draw=red!66!black
]

\tikzstyle{rc:complexity-label}=[
]

\title {Revisiting Chase Termination for Existential Rules \\
        and their Extension to Nonmonotonic Negation}
\author{ Jean-Fran\c{c}ois Baget \\ {\small INRIA}
\And Fabien Garreau \\ {\small University of Angers}
\And Marie-Laure Mugnier \\ {\small University of Montpellier}
\And Swan Rocher \\ {\small University of Montpellier}}

\begin{document}
\nocopyright


\maketitle

\begin{abstract}
Existential rules have been proposed for representing ontological knowledge,
specifically in the context of Ontology- Based Data Access. Entailment with existential rules is undecidable.
We focus in this paper on conditions that ensure the termination
of a breadth-first forward chaining algorithm known as the chase. Several variants of the chase have been proposed.
In the first part of this paper, we propose a new tool that allows to extend existing acyclicity conditions ensuring chase termination, while keeping
good complexity properties. In the second part, we study the extension to existential rules with nonmonotonic negation under stable
model semantics, discuss the relevancy of the chase variants for these rules and further extend acyclicity results
obtained in the positive case.

\end{abstract}

\section{Introduction}
\emph{Existential rules} (also called Datalog+/-) have been proposed for
representing ontological knowledge, specifically in the context of
Ontology-Based Data Access, that aims to exploit ontological knowledge when
accessing data \cite{cali09,blms09}. These rules allow to assert the
existence of unknown individuals, a feature recognized as crucial for
representing knowledge in an open domain perspective. Existential rules
generalize lightweight description logics, such as DL-Lite and $\mathcal{EL}$
\cite{dl-lite07,baader-b-l05} and overcome some of their limitations by
allowing any predicate arity as well as cyclic structures.

Entailment  with existential rules is known to be undecidable
\cite{beeri-vardi81,chandra-lewis-makowsky81}. Many sufficient conditions for
decidability, obtained by syntactic restrictions, have been exhibited in
knowledge representation and database theory (see e.g., the overview in
\cite{rr-11-m}).  We focus in this paper on conditions that ensure the
termination of a breadth-first forward chaining algorithm, known as the
\emph{chase} in the database literature. Given a knowledge base composed of
data and existential rules, the chase saturates the data by application of
the rules. When it is ensured to terminate, inferences enabled by the rules
can be  materialized in the data, which can then be queried like a classical
database, thus allowing to benefit from any database optimizations technique.
  Several variants of the
chase have been proposed, which differ in the way they deal with redundant
information \cite{fagin-kolaitis-al05,deutsch08,marnette09}. It follows that
they do not behave in the same way with respect to termination.
In the following, when we write \emph{the} chase, we mean one of these variants. Various acyclicity notions have been proposed to ensure
the halting of some chase variants.

Nonmonotonic extensions to existential rules were recently considered in
\cite{cali2009stratified} with stratified negation, \cite{gottlob2012wf} with
well-founded semantics and \cite{magka2013computing} with stable model
semantics. This latter work studies skolemized existential rules  and focuses
on cases where a finite unique model exists.

In this paper, we tackle the following issues:  Can we still extend known acyclicity notions ? Would any chase variant be applicable to existential rules provided with nonmonotonic negation, a useful feature for ontological modeling?

\emph{1. Extending acyclicity notions.} Acyclicity conditions can be classified into two main families: the first one constrains the way existential variables are propagated during the chase (e.g.  \cite{fagin03,fagin-kolaitis-al05,marnette09,kr11}) and the second one encodes dependencies between rules, i.e., the fact that a rule may lead to trigger another rule (e.g. \cite{baget04,deutsch08,blms11}). These conditions are based on different graphs, but all of them can be seen as can  as forbidding ``dangerous'' cycles in the considered graph. We define a new family of graphs that allows to extend these acyclicity notions, while keeping good complexity properties.

\emph{2. Processing rules with nonmonotonic negation. }
 We define a notion of stable models on nonmonotonic existential rules and provide
a derivation algorithm  that instantiate rules ``on the fly''
\cite{asperix09,dao2012omiga}. This algorithm is parametrized by a chase
variant. We point out that, differently to the positive case, not all
variants of the chase lead to sound procedures in presence of nonmonotonic
negation; furthermore, skolemizing existential variables or not makes a
semantic difference, even when both computations terminate. Finally, we
further extend acyclicity results obtained on positive rules by exploiting
negative information as well.

A technical report with the proofs omitted for space restriction reasons is available \url{http://www2.lirmm.fr/~baget/publications/nmr2014-long.pdf}.

\section{Preliminaries}
\label{sec:preliminaries}
\subsubsection{Atomsets}
We consider first-order vocabularies with constants but no other function symbols. An \emph{atom} is of the form $p(t_1, \ldots, t_k)$ where $p$ is a predicate of arity $k$ and the $t_i$ are terms, i.e., variables or constants (in the paper we denote constants by $a, b, c, ...$ and variables by $x, y, z, ...$). An \emph{atomset}  is a set of atoms. Unless indicated otherwise, we will always consider \emph{finite} atomsets. If $F$ is an atom or an atomset, we write $\fun{terms}(F)$ (resp. $\fun{vars}(F)$, resp. $\fun{csts}(F)$) the set of terms (resp. variables, resp. constants) that occur in $F$.
 If $F$ is an atomset, we write $\phi(F)$ the formula obtained by the conjunction of all atoms in $F$, and $\Phi(F)$ the existential closure of $\phi(F)$. We say that an atomset $F$ \emph{entails} an atomset $Q$ (notation $F \models Q$) if $\Phi(F) \models \Phi(Q)$. It is well-known that $F \models Q$ iff there exists a \emph{homomorphism} from $Q$ to $F$, i.e., a \emph{substitution} $\sigma: \fun{vars}(F) \rightarrow \fun{terms}(Q)$ such that $\sigma(Q) \subseteq F$. Two atomsets $F$ and $F'$ are said to be \emph{equivalent} if $F \models F'$ and $F' \models F$.
If there is a homomorphism $\sigma$ from an atomset $F$ to itself (i.e., an \emph{endomorphism} of $F$) then $F$ and $\sigma(F)$ are equivalent.
An atomset $F$ is a \emph{core} if there is no homomorphism from $F$ to one of its strict subsets. Among all atomsets equivalent to an atomset $F$, there exists a unique core (up to isomorphism). We call this atomset \emph{the} core of $F$.

\subsubsection{Existential Rules}

An \emph{existential rule} (and simply a rule hereafter) is of the form $B \rightarrow H$, where $B$ and $H$ are atomsets, respectively called the \emph{body} and  the  \emph{head} of the rule.
To an existential rule $R : B \rightarrow H$ we assign a formula $\Phi(R) = \forall \vec{x}\forall \vec{y} (\phi(B) \rightarrow \exists \vec{z} \phi(H))$, where $\fun{vars}(B) = \vec{x} \cup \vec{y}$, and $\fun{vars}(H) = \vec{x} \cup \vec{z}$. Variables $\vec{x}$, which appear in both $B$ and $H$, are called \emph{frontier variables}, while variables $\vec{z}$, which appear only in $H$ are called \emph{existential variables}. E.g.,  $\Phi(b(x,y) \rightarrow h(x,z)) = \forall x \forall y(b(x,y) \rightarrow \exists z h(x,z))$. The presence of existential variables in rule heads is the distinguishing feature of existential rules.

 A \emph{knowledge base} is a pair $K = (F, \mathcal R)$ where $F$ is an atomset (the set of facts) and $\mathcal R$ is a finite set of existential rules. We say that $K = (F, \{R_1, \ldots, R_k\})$ \emph{entails} an atomset $Q$ (notation $K \models Q$) if $\Phi(F), \Phi(R_1), \ldots, \Phi(R_k) \models \Phi(Q)$. The fundamental problem we consider, denoted by {\sc entailment}, is the following: given a knowledge base $K$ and an atomset $Q$, is it true that $K \models Q$? When $\Phi(Q)$ is seen as a Boolean conjunctive query, this problem is exactly the problem of determining if $K$ yields a positive answer to this query.


A rule $R : B \rightarrow H$ is \emph{applicable} to an atomset $F$ if there is a homomorphism $\pi$ from $B$ to $F$. Then the \emph{application of $R$ to $F$ according to $\pi$} produces an atomset $\alpha(F, R, \pi) = F \cup \pi(\fun{safe}(H))$, where $\fun{safe}(H)$ is obtained from $H$ by replacing existential variables with fresh ones.  An $\mathcal R$-derivation from $F$ is a (possibly infinite) sequence $F_0 = \sigma_0(F), \ldots, \sigma_k(F_k), \ldots$ of atomsets such that $\forall 0 \leq i$, $\sigma_i$ is an endomorphism of $F_i$ (that will be used to remove redundancy in $F_i$) and $\forall 0 < i$, there is  a rule $(R: B \rightarrow H) \in \mathcal R$ and a homomorphism $\pi_i$ from $B$ to $\sigma_i(F_{i-1})$ such that $F_i = \alpha(\sigma_i(F_{i-1}),R, \pi_i)$.

\begin{example} Consider the existential rule $R = human(x) \rightarrow hasParent(x, y), human(y)$; and the atomset $F = \{human(a)\}$. The application of $R$ to $F$ produces an atomset $F' = F \cup \{hasParent(x, y_0), human(y_0)\}$ where $y_0$ is a fresh variable denoting an unknown individual. Note that $R$ could be applied again to $F'$ (mapping $x$ to $y_0$), which would create another existential variable and so on.
\end{example}

A finite $\mathcal R$-derivation $F_0, \ldots, F_k$ from $F$  is said to be \emph{from} $F$ \emph{to} $F_k$. Given a knowledge base $K = (F, \mathcal R)$, $K \models Q$ iff there exists a finite $\mathcal R$-derivation from $F$ to $F'$ such that $F' \models Q$ \cite{blms11}.




Let $R_i$ and $R_j$ be rules, and $F$ be an atomset such that $R_i$ is applicable to $F$ by a homomorphism $\pi$;  a homomorphism $\pi'$ from $B_j$ to $F' = \alpha(F, R_i, \pi)$ is said to be \emph{new} if $\pi'(B_j)  \nsubseteq F$.
Given a rule $R = B \rightarrow H$, a homomorphism $\pi$ from $B$ to $F$ is said to be \emph{useful} if it cannot be extended to a homomorphism from $B \cup H$ to $F$; if $\pi$ is not useful then $\alpha(F, R, \pi)$ is equivalent to $F$, but this is not a necessary condition for $\alpha(F, R, \pi)$ to be equivalent to $F$.

\section{Chase Termination}
\label{sec:chase}
 An algorithm that computes an $\mathcal R$-derivation by exploring all possible rule applications in a breadth-first manner is called a \emph{chase}.
 In the following, we will also call chase the derivation it computes. Different kinds of chase can be defined by using different properties to
 compute $F'_i = \sigma_i(F_i)$ in the derivation (hereafter we write $F'_i$ for $\sigma_i(F_i)$  when there is no ambiguity). All these algorithms are sound and complete w.r.t. the {\sc entailment} problem in the sense that $(F, \mathcal R) \models Q$ iff they provide in finite (but unbounded) time a finite $\mathcal R$-derivation from $F$ to $F_k$ such that $F_k \models Q$.

\subsubsection{Different kinds of chase}

In the \emph{oblivious chase}  (also called naive chase), e.g., \cite{cali-gottlob-kifer08},
a rule $R$ is applied according to
a homomorphism $\pi$ only if it has not already been applied according to the same homomorphism.  Let $F_i =
\alpha(F'_{i-1}, R, \pi)$, then $F'_i = F'_{i-1}$ if $R$ was previously applied according to $\pi$, otherwise $F'_i =
F_i$.
This can be slightly improved. Two applications $\pi$ and $\pi'$ of the same rule add the same atoms if they map frontier variables identically (for any frontier variable $x$ of $R$, $\pi(x) = \pi'(x)$); we say that they are frontier-equal. In the \emph{frontier chase}, let $F_i = \alpha(F'_{i-1}, R, \pi)$, we take $F'_i = F'_{i-1}$ if $R$ was previously applied according to some $\pi'$ frontier-equal to $\pi$, otherwise $F'_i = F_i$.
 The \emph{skolem chase} \cite{marnette09} relies on a skolemisation of the rules: a rule $R$ is transformed into a rule \emph{skolem}($R$) by replacing each occurrence of an existential variable $y$ with a functional term $f^R_y(\vec{x})$, where $\vec x$ are the frontier variables of $R$. Then the oblivious chase is run on skolemized rules.
It can easily be checked that frontier chase and skolem chase yield isomorphic results, in the sense that they generate exactly the same atomsets, up to a bijective renaming of variables by skolem terms.

 The \emph{restricted chase} (also called standard chase) \cite{fagin-kolaitis-al05} detects a kind of local redundancy. Let $F_i = \alpha(F'_{i-1}, R, \pi)$, then $F'_i = F_i$ if $\pi$ is useful, otherwise $F'_i = F'_{i-1}$.
%
 The \emph{core chase} \cite{deutsch08} considers the strongest possible form of redundancy: for any $F_i$, $F'_i$ is the core of $F_i$.


A chase is said to be \emph{local} if $\forall i \leq j$, $F'_i \subseteq F'_j$. All chase variants presented above are local,  \emph{except for the core chase}. This property will be critical for nonmonotonic existential rules.

\subsubsection{Chase termination}

Since {\sc entailment} is undecidable, the chase may not halt. We call \emph{$C$-chase} a chase relying on some criterion $C$ to generate $\sigma(F_i) = F'_i $. So $C$ can be oblivious, skolem, restricted, core or any other criterion that ensures the equivalence between $F_i$ and $F'_i$.
 A $C$-chase generates a possibly infinite $\mathcal R$-derivation $\sigma_0(F), \sigma_1(F_1), \ldots, \sigma_k(F_k), \ldots$

We say that this derivation \emph{produces} the (possibly infinite) atomset $(F, \mathcal R)^C = \cup_{0 \leq i \leq \infty} \sigma_i(F_i) \setminus \cup_{0 \leq i \leq \infty} \overline{(\sigma_i(F_i))}$, where $\overline{(\sigma_i(F_i))} = F_i \setminus \sigma(F_i)$. Note that this produced atomset is usually defined as the infinite union of the  $\sigma_i(F_i)$. Both definitions are equivalent when the criterion $C$ is \emph{local}. But the usual definition would produce too big an atomset with a non-local chase such as the core chase: an atom generated at step $i$ and removed at step $j$ would still be present in the infinite union.
 We say that a (possibly infinite) derivation obtained by the $C$-chase is \emph{complete} when any further rule application on that derivation would produce the same atomset. A complete derivation obtained by any $C$-chase produces a \emph{universal model} (i.e., most general) of $(F, \mathcal R)$: for any atomset $Q$, we have $F, \mathcal R \models Q$ iff $(F, \mathcal R)^C \models Q$.

We say that the $C$-chase \emph{halts} on $(F, \mathcal R)$ when the $C$-chase generates a finite complete $\mathcal R$-derivation from $F$ to $F_k$.
Then $(F, \mathcal R)^C = \sigma_k(F_k)$ is a finite universal model.  
We say that $\mathcal R$ is \emph{universally $C$-terminating} when the $C$-chase halts on $(F, \mathcal R)$ for any atomset $F$. We call \emph{$C$-finite} the class of universally $C$-terminating sets of rules. It is well known that the chase variants do not behave in the same way w.r.t. termination. The following examples highlight these different behaviors.

\begin{example} [Oblivious / Skolem chase]
Let $R = p(x,y) \rightarrow p(x,z)$ and $F = \{p(a,b)\}$. The oblivious chase does not halt: it adds $p(a,z_0)$, $p(a,z_1)$, etc. 
The skolem chase considers the rule $p(x,y) \rightarrow p(x,f^R_z(x))$; it adds $p(a, f^R_y(a))$ then halts.
\end{example}

\begin{example}[Skolem / Restricted chase]\label{ex-fes}
Let $R : p(x) \rightarrow r(x,y), r(y,y), p(y)$ and $F = \{p(a)\}$. The skolem chase does not halt: at Step 1, it maps $x$ to $a$ and adds $r(a, f^R_y(a))$, $r(f^R_y(a),f^R_y(a))$ and $p(f^R_y(a))$; at step 2, it maps $x$ to $f^R_y(a)$ and adds $r(f^R_y(a), f^R_y(f^R_y(a)))$, etc. The restricted chase performs a single rule application, which adds $r(a, y_0)$, $r(y_0,y_0)$ and $p(y_0)$; indeed, the rule application that maps $x$ to $y_0$ yields only redundant atoms w.r.t. $r(y_0,y_0)$ and $p(y_0)$.
\end{example}

\begin{example} [Restricted / Core chase]
Let $F = s(a)$, $R_1 = s(x) \rightarrow p(x,x_1),  p(x,x_2), r(x_2,x_2)$, $R_2 = p(x,y) \rightarrow q(y)$ and $R_3 = q(x) \rightarrow r(x,y), q(y)$. Note that $R_1$ creates redundancy and $R_3$ could be applied indefinitely if it were the only rule. $R_1$ is the first applied rule, which creates new variables, called $x_1$ and $x_2$ for simplicity. The restricted chase does not halt: $R_3$ is not applied on $x_2$ because it is already satisfied at this point, but it is applied on $x_1$, which creates an infinite chain. The core chase applies $R_1$,  computes the core of the result, which removes $p(a,x_1)$, then halts.
\end{example}

It is natural to consider the oblivious chase as the weakest form of chase
and necessary to consider the core chase as the strongest form of chase (since the core is the minimal representative of its equivalence class).
 We say that a criterion $C$ is \emph{stronger} than $C'$ and write $C \geq C'$ when $C'$-finite $\subseteq$ $C$-finite. We say that $C$ is
\emph{strictly stronger} than $C'$ (and write $C > C'$) when $C \geq C'$ and $C' \not\geq C$.

It is well-known that core $>$ restricted $>$ skolem $>$ oblivious.
%
%
An immediate remark is that core-finite corresponds to \emph{finite expansion sets} \emph{(fes)} defined in
\cite{baget-mugnier02}.
 To sum up, the following inclusions hold between $C$-finite classes: oblivious-finite $\subset$ skolem-finite = frontier-finite $\subset$ restricted-finite
$\subset$ core-finite = fes.

\section{Known Acyclicity Notions}
\label{sec:review}
We can only give a brief overview of known  acylicity notions, which should however allow to place our contribution
within the existing landscape. A comprehensive taxonomy can be found in \cite{grau2013acyclicity}.

Acyclicity notions ensuring that some chase variant terminates can be divided into two main families, each of them relying on a different graph: a ``position-based'' approach which basically relies on a graph encoding variable sharing between positions in predicates and a ``rule dependency approach'' which relies on a graph encoding dependencies between rules, i.e., the fact that a rule may lead to trigger another rule (or itself).

\subsubsection{Position-based approach}

In the position-based approach, cycles identified as dangerous are those passing through positions that may contain
existential variables; intuitively, such a cycle means that the creation of an existential variable in a given
position may lead to create another existential variable in the same position, hence an infinite number of existential
variables. Acyclicity is then defined by the absence of dangerous cycles.  The simplest notion of acyclicity in this
family is that of  \emph{weak acyclicity} \emph{(wa)} \cite{fagin03} \cite{fagin-kolaitis-al05}, which has been widely used
in databases. It relies on a directed graph whose nodes are the positions in predicates (we denote by $(p,i)$ the position $i$ in predicate $p$).
 Then, for each rule $R: B \rightarrow H$ and each variable $x$ in $B$ occurring in position $(p,i)$, edges with origin $(p,i)$ are built as follows:
 if $x$ is a frontier variable, there is an edge from $(p,i)$ to each position of $x$ in $H$;
 furthermore, for each existential variable $y$ in $H$ occurring in position $(q,j)$, there is a special edge from $(p,i)$ to $(q,j)$.
 A set of rules is weakly acyclic if its associated graph has no cycle passing through a special edge.

 \begin{example}[Weak-acyclicity]\label{ex-wa}
Let $R_1 = h(x) \rightarrow p(x,y)$, where $y$ is an existential variable, and
$R_2 =  p(u,v), q(v) \rightarrow h(v)$. The position graph of $\{R_1, R_2\}$ contains a special edge from $(h,1)$ to $(p,2)$ due to $R_1$ and an edge from $(p,2)$ to $(h,1)$ due to $R_2$, thus $\{R_1, R_2\}$ is not wa.
\end{example}

 Weak-acyclicity  has been generalized, mainly by shifting the focus from positions to existential variables
(\emph{joint-acyclicity} \emph{(ja)}\cite{kr11}) or to positions in atoms instead of predicates
(\emph{super-weak-acyclicity} \emph{(swa)} \cite{marnette09}). Other related notions can be imported from logic programming,
e.g., \emph{finite domain} \emph{(fd)} \cite{fd2008}  and \emph{ argument-restricted} \emph{(ar)} \cite{ar2009}.
See the first column in
Figure \ref{fig:gen}, which shows the inclusions between the corresponding classes of rules (all these inclusions are known to be strict).

\subsubsection{Rule Dependency}

In the second approach, the aim is to avoid cyclic triggering of rules \cite{baget04,blms09,deutsch08,grau2012acyclicity}.
 We say that a rule $R_2$ \emph{depends} on a rule $R_1$ if there exists an atomset $F$ such that $R_1$ is applicable to $F$ according to a homomorphism $\pi$ and $R_2$ is applicable to $F' = \alpha(F, R_1, \pi)$ according to a new useful homomorphism.
This abstract dependency relation can be effectively computed with a unification operation known as piece-unifier
\cite{blms09}.  Piece-unification  takes existential variables into account, hence is more complex than the usual
unification between atoms. A \emph{piece-unifier} 
of a rule body $B_2$ with a
rule head $H_1$ is a substitution $\mu$ of $\fun{vars}(B'_2) \cup \fun{vars}(H'_1)$, where $B'_2 \subseteq B_2$ and $H'_1
\subseteq H_1$, such that \emph{(1) }$\mu(B'_2) = \mu(H'_1)$ and \emph{(2)} existential variables in $H'_1$ are not
unified with separating variables of $B'_2$, i.e., variables that occur both in $B'_2$ and in $(B_2 \setminus B'_2)$; in
other words, if a variable $x$ occuring in $B'_2$ is unified with an existential variable $y$ in $H'_1$, then all atoms
in which $x$ occurs also belong to $B'_2$. It holds that $R_2$ depends on $R_1$ iff there is a piece-unifier of $B_2$ with
$H_1$ satisfying easy to check additional conditions (atom erasing \cite{blms11} and usefulness
\cite{grau2013acyclicity}).

\begin{example}[Rule dependency]\label{ex-dependency}
Consider the rules from Example  \ref{ex-wa}. There is no piece-unifier of $B_2$ with $H_1$. The substitution $\mu = \{(u,x), (v,y)\}$, with $B'_2 = p(u,v)$  and $H'_1 = H_1$, is not a piece-unifier because $v$ is unified with an existential variable, whereas it is a separating variable of $B'_2$ (thus, $q(v)$ should be included in $B'_2$, which is impossible). Thus $R_2$ does not depend on $R_1$.
\end{example}

The \emph{graph of rule dependencies } of a set of rules $\mathcal R$, denoted by GRD($\mathcal R)$, encodes the dependencies between rules in $\mathcal R$. It is a directed graph with set of nodes $\mathcal R$ and an edge $(R_i, R_j)$ if $R_j$ depends on $R_i$ (intuition: ``$R_i$ may lead to trigger $R_j$ in a new way''). E.g., considering the rules in Example \ref{ex-dependency}, the only edge is $(R_2,R_1)$.

When the GRD is acyclic (\emph{aGRD}, \cite{baget04}), any derivation sequence is necessarily finite. This
notion is incomparable with those based on positions.

We point out here that the \emph{oblivious} chase may not stop on \emph{wa} rules.
 Thus, the only acyclicity notion in Figure \ref{fig:gen} that ensures the termination of the
 oblivious chase is  \emph{aGRD} since all other notions generalize \emph{wa}.

\subsubsection{Combining both approches}

 Both approaches have their weaknesses: there may be a dangerous cycle on positions but no cycle w.r.t. rule dependencies (see the preceeding examples), and there may be a cycle w.r.t. rule dependencies whereas rules contain no existential variables (e.g. $p(x,y) \rightarrow p(y,x), q(x)$).
Attempts to combine both notions only succeded to combine them in a ``modular way'':
if the rules in each strongly connected component (s.c.c.) of the GRD
belong to a \emph{fes} class, then the set of rules is \textit{fes} \cite{baget04,deutsch08}.
More specifically, it is easy to check that if for a given $C$-chase, each s.c.c. is $C$-finite, then the $C$-chase
stops.

In this paper, we propose an ``integrated'' way of combining both approaches, which relies on a single graph. This
allows to unify preceding results and to generalize them without complexity increasing (the new acyclicity notions are
those with a gray background in Figure \ref{fig:gen}).

Finally, let us mention \emph{model-faithful acyclicity}\emph{ (mfa)} \cite{grau2012acyclicity}, which
generalizes the previous acyclicity notions and cannot be captured by our approach. Briefly, mfa involves running the skolem chase
until termination or a cyclic functional term is found. The price to pay for the generality of this property is high complexity:
checking if a set of rules is universally mfa (i.e., for any set of facts) is 2EXPTIME-complete. Checking
\emph{model-summarizing acyclicity (msa)}, which approximates mfa, remains EXPTIME-complete. In contrast, checking position-based
properties is in PTIME and checking agrd is also co-NP-complete. Sets of rules satisfying mfa are skolem-finite
\cite{grau2012acyclicity}, thus all
properties studied in this paper ensure $C$-finiteness, when $C \geq$ skolem.

\begin{center}
\begin{figure}[t]

\begin{tikzpicture}[node distance=1cm]
	\node[fake] (contrib0) at (1.9,1.75) {};
	\node[fake] (contrib1) at (4.35,1.75) {};
	\node[fake] (contrib2) at (4.35,8.35) {};
	\node[fake] (contrib3) at (1.9,8.35) {};
	\draw[draw=none,fill=black!10] (contrib0.center) -- (contrib1.center) -- (contrib2.center) -- (contrib3.center) --
	(contrib0.center);

	\node[rc:node] (wa) at (0,0) {$\wa$};

	\node[rc:node,node distance=1.25cm] (agrd) [right of=wa] {$\agrd$};
	\node[rc:node] (wad) [above of=agrd] {$\grd{\wa}$};

	\node[fake,node distance=1.25cm] (f0) [right of=wad] {};
	\node[rc:node] (wau) [above of=f0] {$\gerd{\wa}$};

	\node[fake,node distance=1.25cm] (f1) [right of=wau] {};
	\node[rc:node] (wac) [above of=f1] {$\gerdc{\wa}$};

	\foreach \t/\g in {/,d/\grd,u/\gerd,c/\gerdc}{
		\node[rc:node] (fd\t) [above of=wa\t]   {$\g{\fd}$};
		\node[rc:node] (ar\t) [above of=fd\t]   {$\g{\ar}$};
		\node[rc:node] (ja\t) [above of=ar\t]   {$\g{\ja}$};
		\node[rc:node] (swa\t) [above of=ja\t]  {$\g{\swa}$};
		\node[rc:node] (msa\t) [above of=swa\t] {$\g{\msa}$};
	}

	\node[rc:node] (mfa) [above of=msac] {$\mfa$};

	\foreach \c in {wa,fd,ar,ja,swa,msa} {
		\draw[rc:edge] (\c) -- (\c d);
		\draw[rc:edge] (\c d) -- (\c u);
		\draw[rc:edge] (\c u) -- (\c c);
	}

	\foreach \t in {,d,u,c} {
		\draw[rc:edge] (wa\t) -- (fd\t);
		\draw[rc:edge] (fd\t) -- (ar\t);
		\draw[rc:edge] (ar\t) -- (ja\t);
		\draw[rc:edge] (ja\t) -- (swa\t);
		\draw[rc:edge] (swa\t) -- (msa\t);
	}

	\draw[rc:edge] (agrd) -- (wad);
	\draw[rc:edge] (msac) -- (mfa);

	\node[fake,node distance=0.5cm
	] (p0) [below left of=wa] {};
	\node[fake,node distance=0.5cm
	] (p1) [below right of=wa] {};
	\node[fake,node distance=0.5cm] (p2) [above right of=swa] {};
	\node[fake,node distance=0.5cm] (p3) [above left of=swa] {};
	\draw[rc:P-edge] (p0) -- (p1) -- (p2) -- (p3) -- (p0);

	\node[fake] (n0) at (0.75,-0.25) {};
	\node[fake] (n1) at (1.7,-0.25) {};
	\node[fake] (n2) at (4.25,2) {};
	\node[fake] (n3) at (4.25,7.25) {};
	\node[fake] (n4) at (3.5,7.25) {};
	\node[fake] (n5) at (0.75,5) {};
	\draw[rc:NP-edge] (n0) -- (n1) -- (n2) -- (n3) -- (n4) -- (n5) -- (n0);

	\node[fake] (e0) at (-0.4,4.75) {};
	\node[fake] (e1) at (0.25,4.75) {};
	\node[fake] (e2) at (4.25,8) {};
	\node[fake] (e3) at (4.25,8.25) {};
	\node[fake] (e4) at (3.3,8.25) {};
	\node[fake] (e5) at (-0.4,5.25) {};
	\draw[rc:ET-edge] (e0) -- (e1) -- (e2) -- (e3) -- (e4) -- (e5) -- (e0);

	\node[fake] (ee0) at(3.3,8.75)  {};
	\node[fake] (ee1) at(4.25,8.75) {};
	\node[fake] (ee2) at(4.25,9.25) {};
	\node[fake] (ee3) at(3.3,9.25)  {};
	\draw[rc:ETT-edge] (ee0) -- (ee1) -- (ee2) -- (ee3) -- (ee0);

	\node[rc:complexity-label, node distance=0.6cm] [left of=fd] {{\color{green!50!black}$P$}};
	\node[rc:complexity-label] [right of=wad] {{\color{blue!50!black}$coNP$}};
	\node[rc:complexity-label] [above of=msad] {{\color{orange!50!black}$Exp$}};
	\node[rc:complexity-label] [left of=mfa] {{\color{red!50!black}$2$-$Exp$}};


\end{tikzpicture}
	\caption{Relations between recognizable acyclicity properties. All inclusions are strict and complete (i.e., if there is no path between two properties then they are incomparable). }
\label{fig:gen}
\end{figure}
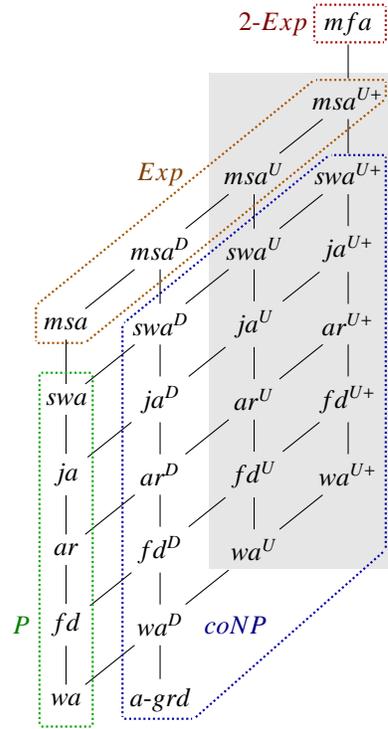
\end{center}

\section{Extending Acyclicity Notions}
\label{sec:acyclicity}



In this section, we combine rule dependency and propagation of existential
variables into a single graph. W.l.o.g. we assume that distinct rules do not
share any variable. Given an atom $a = p(t_1, \dots, t_k)$, the $i^{th}$
position in $a$ is denoted by $\vpos{a}{i}$, with $\fun{pred}(\vpos{a}{i}) =
p$ and $\fun{term}(\vpos{a}{i}) = t_i$. If $A$ is an atomset such that $a \in
A$, we say that $\vpos{a}{i}$ is in $A$.
 If $\fun{term}(\vpos{a}{i})$ is an existential (resp. frontier) variable, $\vpos{a}{i}$
is called an \emph{existential} (resp. \emph{frontier}) position. In the
following, we use ``position graph'' as a generic name to denote a graph
whose nodes are positions in \emph{atoms}.


We first define the notion of a basic position graph, which takes each rule
in isolation.  Then, by adding edges to this graph, we define three position
graphs with increasing expressivity, i.e., allowing to check termination for
increasingly larger classes of rules.

\begin{definition}[(Basic) Position Graph ($\pg$)] The \emph{position graph} of a rule $R : B \rightarrow H$ is the directed graph $\PG{R}$ defined as follows:
	\begin{itemize}
    \item there is a node for each $\vpos{a}{i}$ in $B$ or in $H$;
    \item for all frontier positions $\vpos{b}{i} \in B$ and all
        $\vpos{h}{j} \in H$, there is an edge from $\vpos{b}{i}$
	to $\vpos{h}{j}$ if $\fun{term}(\vpos{b}{i}) = \fun{term}(\vpos{h}{j})$ or if $\vpos{h}{j}$ is existential.

	\end{itemize}
Given a set of rules $\R$, the \emph{basic position graph} of $\R$, denoted
by $\PG{\R}$,  is the disjoint union of $\PG{R_i}$, for all $R_i \in \R$.
\end{definition}

An existential position $\vpos{a}{i}$ is said to be \emph{infinite} if there
is an atomset $F$ such that running the chase on $F$ produces an unbounded
number of instantiations of $\fun{term}(\vpos{a}{i})$. To detect infinite
positions, we encode how variables may be ``propagated" among rules by adding
 edges to $\PG{\R}$, called \emph{transition edges}, which go from positions in rule heads to positions  in rule bodies.
 The set of transition edges has to be \emph{correct}: if an existential position
$\vpos{a}{i}$ is infinite, there must be a cycle going through $\vpos{a}{i}$
in the graph.

We now define three position graphs by adding transition edges to $\PG{\R}$,
namely $\FPG{\R}$, $\DPG{\R}$ and $\UPG{\R}$. All three graphs have correct
sets of edges. Intuitively, $\FPG{\R}$ corresponds to the case where all
rules are supposed to depend on all rules; its set of cycles is in bijection
with the set of cycles in the predicate position graph defining
weak-acyclicity. $\DPG{\R}$ encodes actual paths of rule dependencies.
Finally, $\UPG{\R}$ adds information about the piece-unifiers themselves.
This provides an accurate encoding of variable propagation from an atom
position to another.

\begin{definition}[$\xpg$]\label{def-xpg}
	Let $\R$ be a set of rules.
	The three following position graphs are obtained from $\PG{\R}$ by adding a (transition)
	edge from each $k^{th}$ position $\vpos{h}{k}$ in a rule head $H_i$ to each $k^{th}$ position  $\vpos{b}{k}$ in
	a rule body $B_{j}$, with the  \emph{same predicate},
	provided that some condition is satisfied :
	\begin{itemize}
		\item {\em full PG}, denoted by $\FPG{\R}$: no additional condition;
		\item {\em dependency PG}, denoted by  $\DPG{\R}$: if $R_{j}$ depends directly or indirectly on $R_i$, i.e., if there is a path from $R_i$ to $R_j$ in GRD($\mathcal R)$;
		\item {\em PG with unifiers}, denoted by $\UPG{\R}$: if there is a piece-unifier $\mu$ of $B_j$ with the head of an agglomerated rule $R^j_i$
		         such that $\mu(\fun{term}([b,k])) = \mu(\fun{term}([h,k]))$, where $R^j_i$ is formally defined below (Definition~\ref{def-agg})		
	\end{itemize}
\end{definition}

An agglomerated rule associated with $(R_i, R_j)$ gathers information about
selected piece-unifiers along (some) paths from $R_i$ to (some) predecessors
of $R_j$.


\begin{definition}[Agglomerated Rule] \label{def-agg}
Given $R_i$ and $R_j$ rules from $\mathcal R$, an agglomerated rule
associated with $(R_i, Rj)$ has the following form:
$$R^j_i =  B_i \cup_{t \in T \subseteq \fun{terms}(H_i)} \fun{fr}(t) \rightarrow H_i$$
where $\fun{fr}$ is a new unary predicate that does not appear in $\mathcal
R$, and the atoms $\fun{fr}(t)$ are built as follows. Let $\mathcal P$ be a
non-empty set of paths from $R_i$ to direct predecessors of $R_j$ in
GRD($\mathcal R)$. Let $P= (R_1, \ldots, R_n)$ be a path in $\mathcal P$. One
can associate a rule $R^P$ with $P$ by building a sequence $R_1= R^p_1,
\ldots, R^p_n = R^P$ such that $\forall 1 \leq l < n$, there is a
piece-unifier $\mu_l$ of $B_{l+1}$ with the head of $R^p_l$, where the body
of $R^p_{l+1}$ is $B^p_{l} \cup \{\fun{fr}(t) \, | \, t \, \mbox{is a term of
} \, H^p_l\, \mbox{unified in } \mu_l\}$, and the head of $R^p_{l+1}$ is
$H_1$. Note that for all $l$, $H^p_l = H_1$, however, for $l \neq 1$, $R^p_l$
may have less existential variables than $R_l$ due to the added atoms. The
agglomerated rule $R^j_i$ built from $\{R^P | P \in \mathcal P\}$ is $R^j_i =
\bigcup_{P \in \mathcal P} R^P$.
\end{definition}




\begin{proposition}[Inclusions between $\xpg$]
	Let $\R$ be a set of rules.
	$\UPG{\R} \subseteq \DPG{\R} \subseteq \FPG{\R}$.
	Furthermore, $\DPG{\R} = \FPG{\R}$ if the transitive closure of $GRD(\R)$ is a complete graph.
\end{proposition}

\begin{example}[$PG^F$ and $PG^D$]\label{ex:fpg-dpg}
	Let $\mathcal R = \{R_1, R_2\}$ from Example \ref{ex-wa}.
	Figure~\ref{fig:fpg-dpg} pictures $\FPG{\R}$ and $\DPG{\R}$.
	The dashed edges  belong to $\FPG{\R}$ but not to  $\DPG{\R}$. Indeed, $R_2$ does not depend on $R_1$.
	$\FPG{\R}$ has a cycle 
	while $\DPG{\R}$ has not.
\end{example}


\begin{example}[$PG^D$ and $PG^U$]\label{ex:dpg-upg}
	Let $\R = \{R_1, R_2\}$, with  $R_1 =  t(x,y) \rightarrow p(z,y), q(y)$ and
		$R_2 = p(u,v),q(u) \rightarrow t(v,w)$. In Figure~\ref{fig:dpg-upg}, the dashed edges belong to $\DPG{\R}$ but not to $\UPG{\R}$.
	Indeed,  the only piece-unifier of $B_2$ with $H_1$ unifies $u$ and $y$.
	Hence, the cycle in $\DPG{\R}$ disappears in $\UPG{\R}$.
\end{example}

\begin{center}
\begin{figure}[t]
\input{./figures/expg2.tex}
\caption{$\FPG{\R}$ and $\DPG{\R}$ from Example \ref{ex:fpg-dpg}. Position $\vpos{a}{i}$ is
	represented by underlining the i-th term in $a$.
	Dashed edges do not belong to $\DPG{\R}$.
}
\label{fig:fpg-dpg}
\end{figure}
\end{center}


\begin{center}
\begin{figure}[b]
\input{./figures/expgS.tex}
\caption{$\DPG{\R}$ and $\UPG{\R}$ from Example \ref{ex:dpg-upg}.
	Dashed edges do not belong to $\UPG{\R}$.
}
\label{fig:dpg-upg}
\end{figure}
\end{center}



We now study how acyclicity properties can be expressed on position graphs.
The idea is to associate, with an acyclicity property, a function that
assigns to each position a subset of positions reachable from this position,
according to some propagation constraints;  then, the property is fulfilled
if  no existential position can be reached from itself.
 More precisely, a \emph{marking function} $Y$ assigns to each node $\vpos{a}{i}$
	in a position graph $\xpg$, a subset of its (direct or indirect) successors,  
	called its {\em marking}. 
	A \emph{marked cycle} for $\vpos{a}{i}$ (w.r.t. $X$ and $Y$)
	is a cycle $C$ in $\xpg$ such that
	$\vpos{a}{i} \in C$ and for all $\vpos{a'}{i'} \in C$, $\vpos{a'}{i'}$ belongs to the marking of $\vpos{a}{i}$. 
  Obviously, the less situations there are in which the marking may ``propagate'' in a position graph,
  the stronger the acyclicity property is.


\begin{definition}[Acyclicity property]
	Let $Y$ be a marking function and $\xpg$ be a position graph.
	The\emph{ acyclicity property }associated with $Y$ in $\xpg$,  denoted by $Y^X$, is satisfied if
	there is no marked cycle for an existential position in  $\xpg$. 
         If $Y^X$ is satisfied, we also say that $\XPG{\R}$ \emph{satisfies} $Y$.
\end{definition}

For instance, the marking function associated with weak-acyclicity assigns to
each node the set of its successors in $\FPG{\R}$, without any additional
constraint. The next proposition states that such marking functions can be
defined for each class of rules between $\wa$ and $\swa$
 (first column in Figure~\ref{fig:gen}), in such a way that the associated acyclicity property in $PG^F$ characterizes this class.

\begin{proposition}
\label{prop:markings}
	A set of rules $\R$ is $\wa$  (resp.   $\fd$,  $\ar$,  $\ja$,  $\swa$) iff $\FPG{\R}$ satisfies the acyclicity property
	associated with  $\wa$- (resp.  $\fd$-,  $\ar$-,  $\ja$-,  $\swa$-) marking.
\end{proposition}


As already mentioned, all these classes can be safely extended by combining
them with the GRD. To formalize this, we recall the notion $Y^<$ from
\cite{grau2013acyclicity}: given an acyclicity property $Y$, a set of rules
$\R$ is said to satisfy $Y^<$ if each s.c.c. of $GRD(\R)$ satisfies $Y$,
except for those composed of a single rule with no loop.\footnote{This
particular case is to cover \emph{aGRD}, in which each s.c.c. is an isolated
node.} Whether $\R$ satisfies $Y^<$ can be checked on $\DPG{\R}$:

\begin{proposition}
\label{prop:yd-eq-scc}
	Let $\R$ be a set of rules, and $Y$ be an acyclicity property.
	$\R$ satisfies $Y^<$ iff $\DPG{\R}$ satisfies $Y$,
	i.e., $Y^< = \grd{Y}$.
\end{proposition}


For the sake of brevity, if $Y_1$ and $Y_2$ are two acyclicity properties, we
write $Y_1 \subseteq Y_2$ if any set of rules satisfying $Y_1$ also satisfies
$Y_2$. The following results are straightforward.

\begin{proposition}
	Let $Y_1,Y_2$ be two acyclicity properties.
	If $Y_1 \subseteq Y_2$,
	then $\grd{Y_1} \subseteq \grd{Y_2}$.
\end{proposition}


\begin{proposition}
\label{prop:y-strict-ygrd}
	Let $Y$ be an acyclicity property.
	If $\agrd \nsubseteq Y$ then $Y \subset \grd{Y}$.
\end{proposition}

Hence, any class of rules  satisfying a property $Y^D$ strictly includes both
$\agrd$ and the class characterized by $Y$; (e.g., Figure~\ref{fig:gen}, from
Column 1 to Column 2). More generally, strict inclusion in the first column
leads to strict inclusion in the second one:

\begin{proposition}
\label{prop:prop6}
	Let $Y_1, Y_2$ be two acyclicity properties such that $Y_1 \subset Y_2$, $\wa \subseteq Y_1$ and  $Y_2 \nsubseteq \grd{Y_1}$.
	Then $\grd{Y_1} \subset \grd{Y_2}$.
\end{proposition}

The next theorem states that $\upg$ is strictly more powerful than $\dpg$;
moreover,
	the ``jump" from $\grd{Y}$ to $\gerd{Y}$  is at least as large as
	from $Y$ to $\grd{Y}$.

\begin{theorem}
\label{theo:ygrd-strict-ygerd}
	Let $Y$ be an acyclicity property.
	If $Y \subset \grd{Y}$ then $\grd{Y} \subset \gerd{Y}$.
	Furthermore, there is an injective mapping from the sets of rules satisfying $\grd{Y}$ but not $Y$, to
	the sets of rules satisfying $\gerd{Y}$ but not $\grd{Y}$.
\end{theorem}

\begin{proof}
	Assume $Y \subset \grd{Y}$
	and $\R$ satisfies $\grd{Y}$ but  not $Y$.
	$\R$ can be rewritten into $\R'$ by applying the following steps.
	First, for each rule $R_i = B_i[\X,\Y] \rightarrow H_i[\Y,\Z] \in \R$,
	let $R_{i,1} = B_i[\X,\Y] \rightarrow p_i(\X,\Y)$ where $p_i$ is a fresh
	predicate; and $R_{i,2} = p_i(\X,\Y) \rightarrow H_i[\Y,\Z]$.
	Then, for each rule $R_{i,1}$, let $R'_{i,1}$ be the rule $(B'_{i,1} \rightarrow H_{i,1})$
	with $B'_{i,1} = B_{i,1} \cup \{p'_{j,i}(x_{j,i}): \forall R_{j} \in \R\}$,
	where $p'_{j,i}$ are fresh predicates and $x_{j,i}$ fresh variables.
	Now, for each rule $R_{i,2}$, let $R'_{i,2}$ be the rule $(B_{i,2} \rightarrow H'_{i,2})$
	with $H'_{i,2} = H_{i,2} \cup \{p'_{i,j}(z_{i,j}): \forall R_{j} \in \R\}$,
	where $z_{i,j}$ are fresh existential variables.
	Let $\R' = \bigcup\limits_{R_i \in \R} \{R'_{i,1},R'_{i,2}\}$.
    This construction ensures that each  $R'_{i,2}$ depends on $R'_{i,1}$,
	and each $R'_{i,1}$ depends on each $R'_{j,2}$,
	thus, there is a {\em transition} edge from each $R'_{i,1}$ to $R'_{i,2}$
	and from each $R'_{j,2}$ to each $R'_{i,1}$.
	Hence, $\DPG{\R'}$ contains exactly one cycle for each cycle in $\FPG{\R}$.
	 Furthermore, $\DPG{\R'}$ contains at least one marked cycle w.r.t. $Y$,
	and then $\R'$ does not satisfy $\grd{Y}$.
	Now, each cycle in $\UPG{\R'}$
	is also a cycle in $\DPG{\R}$, and, since $\DPG{\R}$ satisfies $Y$, $\UPG{\R'}$ also does.
    Hence, $\R'$ does not belong to $\grd{Y}$ but to $\gerd{Y}$.
\end{proof}

We also check that strict inclusions in the second column in
Figure~\ref{fig:gen} lead to strict inclusions in the third column.

\begin{theorem} \label{theo:ygrd-strict-ygerd2}
	Let $Y_1$ and $Y_2$ be two acyclicity properties.
	If $\grd{Y_1} \subset \grd{Y_2}$ then $\gerd{Y_1} \subset \gerd{Y_2}$.
\end{theorem}
\begin{proof}
	Let $\R$ be a set of rules such that $\R$ satisfies $\grd{Y_2}$ but
	does not satisfy $\grd{Y_1}$.
	We rewrite $\R$ into $\R'$ by applying the following steps.
	For each pair of rules $R_i,R_j \in \R$ such that there is a dependency path from $R_i$ to $R_j$,
	for each variable $x$ in the frontier of $R_j$ and each variable $y$ in the
	head of $R_i$, if $x$ and $y$ occur both in a given predicate position,
	we add to the body of $R_j$ a new atom $p_{i,j,x,y}(x)$ and to the head of $R_i$
	a new atom $p_{i,j,x,y}(y)$, where $p_{i,j,x,y}$ denotes a fresh predicate.
	This construction allows each term from the head of $R_i$ to propagate to each
	term from the body of $R_j$, if they share some predicate position in $\R$.
	Thus, any cycle in $\DPG{\R}$ is also in $\UPG{\R'}$, without any change in the
	behavior w.r.t. the acyclicity properties.
	Hence $\R'$ satisfies $\gerd{Y_2}$ but does not satisfy $\gerd{Y_1}$.
\end{proof}

The next result states that  $Y^U$ is a sufficient condition for chase
termination:

\begin{theorem} \label{theo:correct}
Let $Y$ be an acyclicity property ensuring the halting of some chase variant
$C$. Then, the $C$-chase halts for any set of rules $\mathcal R$ that
satisfies $Y^U$ (hence $Y^D$). 	
\end{theorem}

\begin{example} Consider again the set of rules $\R$ from Example \ref{ex:dpg-upg}. Figure \ref{fig:dpg-upg}  pictures
 the associated position graphs $\DPG{\R}$ and $\UPG{\R}$. $\R$ is not \emph{aGRD}, nor \emph{wa},
 nor \emph{wa}$^D$ since $\DPG{\R}$ contains a (marked) cycle that goes through the existential position $\vpos{t(v,w)}{2}$.
 However, $\R$ is obviously \emph{wa}$^U$ since $\UPG{\R}$ is acyclic. Hence, the skolem chase and stronger chase variants
 halt for  $\R$ and any set of facts.

\end{example}

Finally, we remind that classes from $\wa$ to $\swa$ can be recognized in
PTIME, and checking $\agrd$ is \coNPc. Hence, as stated by the next result,
the expressiveness gain is without increasing worst-case complexity.

\begin{theorem}[Complexity]
\label{theo:complexity}
	Let $Y$ be an acyclicity property, and $\R$ be a set of rules.
	If checking that $\R$ satisfies $Y$ is in \coNP, then
	checking that $\R$ satisfies $\grd{Y}$ or $\gerd{Y}$ is \coNPc.
\end{theorem}

\section{Further Refinements}
\label{sec:unif}


Still without complexity increasing, we can further extend $Y^U$ into $Y^{U^+}$
by a finer analysis of marked cycles and unifiers.
We define the notion of \emph{incompatible} sequence of unifiers,
which ensures that a given sequence of rule applications is impossible.
Briefly, a marked cycle for which all sequences of unifiers are incompatible
 can be ignored.
 Beside the gain for positive rules,
this refinement will allow one to  take better advantage of negation.

We first point out that the notion of piece-unifier is not appropriate to our purpose.
We have to relax it, as illustrated by the next example.
We call \emph{unifier}, of a rule body $B_2$ with a
rule head $H_1$, a substitution $\mu$ of $\fun{vars}(B'_2) \cup \fun{vars}(H'_1)$,
where $B'_2 \subseteq B_2$ and $H'_1
\subseteq H_1$, such that  $\mu(B'_2) = \mu(H'_1)$
(thus, it satisfies Condition $(1)$ of a piece-unifier).

\begin{example} \label{ex-unif} Let $\mathcal R = \{R_1, R_2, R_3, R_4\}$ with:\\
$R_1 : p(x_1,y_1) \rightarrow q(y_1,z_1)$$~~~$\\
$R_2 : q(x_2,y_2) \rightarrow r(x_2,y_2)$\\
$R_3 : r(x_3,y_3) \wedge s(x_3,y_3) \rightarrow p(x_3,y_3)$ \\ $R_4 : q(x_4,y_4) \rightarrow s(x_4,y_4)$\\ 
There is a dependency cycle $(R_1, R_2, R_3, R_1)$ and a corresponding cycle in $PG^U$.
We want to know if such a sequence of rule applications is possible.
We build the following new rule, which is a composition of $R_1$ and $R_2$ (formally defined later):
$R_1 \unified_\unifier R_2 : p(x_1,y_1) \rightarrow q(y_1,z_1) \wedge r(y_1,z_1)$\\
There is no piece-unifier of $R_3$ with $R_1 \unified_\unifier R_2$, since $y_3$ would be a separating variable mapped to the existential variable $z_1$.
This actually means that $R_3$ is not applicable \emph{right after} $R_1 \unified_\unifier R_2$.
However, the atom needed to apply $s(x_3,y_3)$ can be brought by a sequence of rule applications $(R_1,R_4)$.
We thus relax the notion of piece-unifier to take into account arbitrary long sequences of rule applications.
\end{example}




\begin{definition}[Compatible unifier]
	Let $R_1$ and $R_2$ be rules.
	A unifier $\unifier$ of $B_2$ with $H_1$ is {\em compatible}
	if, for each position $\vpos{a}{i}$ in $B'_2$,
	such that $\unifier(\fun{term}(\vpos{a}{i}))$ is an existential
	variable $z$ in $H'_1$, $\UPG{\R}$ contains a path,
       from a position in which $z$ occurs,
	to $\vpos{a}{i}$, that does not go through another existential position. Otherwise, $\unifier$ is \emph{incompatible}.
\end{definition}

Note that a piece-unifier is necessarily compatible.



\begin{proposition}
\label{prop:inc-unif-noapp}
Let $R_1$ and $R_2$ be rules, and let $\mu$ be a unifier of $B_2$ with $H_1$. If $\mu$ is incompatible, then no application of $R_2$ can use an atom in $\mu(H_1)$.
\end{proposition}


We define the rule corresponding to the composition of $R_1$ and $R_2$ according to a
compatible unifier, then use this notion to define a compatible sequence of unifiers.

\begin{definition}[Unified rule, Compatible sequence of unifiers]
\label{def:unified-rule}$~$\\
$\bullet$ Let $R_1$ and $R_2$ be rules such that there is a compatible unifier $\unifier$
	 of $B_2$ with $H_1$.
	The associated {\em unified rule} $R_\unifier = R_1 \unified_\unifier R_2$ is defined by $H_\unifier = \unifier(H_1) \cup \unifier(H_2)$,
		 and $B_\unifier = \unifier(B_1) \cup (\unifier(B_2) \setminus \unifier(H_1))$.\\
	$\bullet$ Let $(R_1, \ldots, R_{k+1})$ be a sequence of rules. A sequence $s = (R_1~\mu_1 ~R_2 \ldots ~\mu_{k}~R_{k+1})$,
	where, for $1 \leq i \leq k$,  $\unifier_i$ is a unifier of $B_{i+1}$ with $H_i$,
	 is  a \emph{compatible sequence} of unifiers if: \emph{(1) } $\mu_1$ is a compatible unifier of $B_2$ with $H_1$, and\emph{ (2) }
	 if $k > 0$, the sequence obtained from $s$ by replacing $(R_1 ~\mu_1~R_2)$ with
	$R_1 \unified_{\unifier_1} R_{2}$ is a \emph{compatible} sequence of unifiers.
\end{definition}

E.g., in Example~\ref{ex-unif}, the sequence $(R_1~\mu_1~R_2~\mu_2~R_3~\mu_3~R_1)$, with the obvious $\mu_i$, is compatible.
 We can now improve all previous acyclicity properties (see the fourth column in Figure~\ref{fig:gen}).

\begin{definition}[Compatible cycles]
	Let $Y$ be an acyclicity property, and $\upg$ be a position graph with unifiers.
	The {\em compatible cycles for $\vpos{a}{i}$} in $\upg$ are all marked cycles $C$ for $\vpos{a}{i}$ wrt $Y$,
	such that there is a compatible sequence of unifiers induced by $C$.
	Property $\gerdc{Y}$is satisfied if, for each existential position $\vpos{a}{i}$, there is no compatible cycle for $\vpos{a}{i}$ in $\upg$.
\end{definition}

Results similar to Theorem \ref{theo:ygrd-strict-ygerd} and Theorem  \ref{theo:ygrd-strict-ygerd2} are obtained for $Y^{U^+}$ w.r.t. $Y^U$, namely:
\begin{itemize}
\item For any acyclicity property $Y$, $\gerd{Y} \subset \gerdc{Y}$.
\item For any acyclicity properties $Y_1$ and $Y_2$, if $\gerd{Y_1} \subset \gerd{Y_2}$, then $\gerdc{Y_1} \subset \gerdc{Y_2}$.
\end{itemize}

Moreover, Theorem \ref{theo:correct} can be extended to $Y^{U^+}$:  let $Y$ be an acyclicity property ensuring the halting of some chase variant $C$;
then the $C$-chase halts for any set of rules $\mathcal R$ that satisfies $Y^{U^+}$ (hence $Y^U$).
Finally, the complexity result from Theorem \ref{theo:complexity} still holds
 for this improvement.

\section{Handling Nonmonotonic Negation}
\label{sec:negation}
We now add nonmonotonic negation, which we denote by \textbf{not}. A \emph{nonmonotonic existential rule} (NME rule) $R$ is of the form $(B^+, \fun{\bf not} B^-_1, \ldots, \fun{\bf not} B^-_k \rightarrow H)$, where $B^+$, $B^-_i$ and $H$ are atomsets, respectively called the \emph{positive} body, the \emph{negative} bodies and the head of $R$. Note that we generalize the usual notion of negative body by allowing to negate conjunctions of atoms. Moreover, the rule head may contain several atoms.
However, we  impose a safeness condition: $\forall 1 \leq i \leq k$, $\fun{vars}(B^-_i) \subseteq \fun{vars}(B^+)$.
The formula assigned to $R$ is  $\Phi^{not}(R) = \forall \vec{x}\forall \vec{y} (\phi(B^+) \wedge \fun{\bf not}\phi(B^-_1), \ldots, \fun{\bf not}\phi(B^-_k) \rightarrow \exists \vec{z} \phi(H)$.
We write $\fun{pos}(R)$ the existential rule obtained from $R$ by removing its negative bodies, and $\fun{pos}(\mathcal R)$ the set of all $\fun{pos}(R)$ rules, for $R \in \mathcal R$.

\subsubsection{About our Stable Model Semantics}

Answer Set Programming \cite{gelfond2007} introduced stable model semantics for propositional logic, and was naturally extended to grounded programs ({\it i.e.,} sets of NME rules without variables). In this framework, the semantics can be provided through the Gelfond-Lifschitz reduct operator that allows to compute a saturation ({\it i.e.,} a chase) using only grounded NME rules. This semantics can be easily extended to rules with no existential variable in the head, or to skolemized NME rules, as done, for instance, in \cite{magka2013computing}. The choice of the chase/saturation mechanism is here irrelevant, since no such mechanism can produce any redundancy.

The problem comes when considering existential variables in the head of rules. Several semantics have been proposed in that case, for instance circumscription in \cite{Ferraris2011}, or justified stable models in \cite{you2013disjunctive}. We have chosen not to adopt circumscription since it translates NME rules to second-order expressions, and thus would not have allowed to build upon results obtained in the existential rule formalism. In the same way, we have not considered justified stable models, whose semantics does not correspond to stable models on grounded rules, as shown by the following example:

\begin{example} Let $\Pi_1 = \{\emptyset \rightarrow p(a);  p(a), \fun{\bf not} \, q(a) \rightarrow t(a).\}$ be a set of ground NME rules. Then $\{p(a); q(a)\}$ is a justified stable model, but not a stable model. Let $\Pi_2 = \{\emptyset \rightarrow p(a); p(a), \fun{\bf not} \, q(b) \rightarrow t(a)\}$ . Then $\{p(a); t(a)\}$ is a stable model but not a justified stable model.
\end{example}

Let us now recast the Gelfond-Lifschitz reduct-based semantics  in terms of the skolem-chase. Essentially (we will be more precise in the next section), a stable model $M$ is a possibly infinite atomset produced by a skolem-chase that respects some particular conditions:
\begin{itemize}
    \item all rule applications are sound, {\it i.e.,} none of its negative bodies can be found in the stable model produced (the rule is not blocked);
    \item the derivation is complete, {\it i.e.,} any rule applicable and not blocked is applied in the derivation.
\end{itemize}

In the next subsection, we formally define the notion of a stable model, while replacing the skolem-chase with any $C$-chase. We thus obtain a family of semantics parameterized by the considered chase, and  define different notions of $C$-stable models. 


\subsubsection{On the  Chase and Stable Models}

We define a notion of stable model directly on nonmonotonic existential rules and provide a derivation algorithm inspired from the notion of computation in \cite{Liu2010295} and Answer Set Programming solvers that instantiate rules  on the fly \cite{asperix09,dao2012omiga} instead of grounding rules before applying them. The difference with our framework is that they consider normal logic programs, which are a generalization of skolemized NME rules.

%
%
A natural question is then to understand if the choice of a chase mechanism has an impact, not only on the termination,  but also  on the semantics. Thus, we consider the chase as a parameter. Intuitively, a $C$-stable set $A$ is produced by a $C$-chase that, according to \cite{gelfond2007}, must satisfy the NME rules (we say that it is \emph{sound}, {\it i.e.,} that no negative body appearing in the chase is in $A$) and the \emph{rationality principle} (the sound chase does not generate anything that cannot be believed, and it must be complete: any rule application not present in the chase would be unsound).

To define \emph{$C$-stable sets}, we first need to introduce additional notions. A \emph{NME $\mathcal R$-derivation from $F$} is a $\fun{pos}(\mathcal R)$-derivation from $\mathcal R$. This derivation $D = (F_0 = \sigma_0(F), \ldots, \sigma_k(F_k), \ldots)$ produces
a possibly infinite atomset $A$. Let $R$ be a NME rule such that $\fun{pos}(R)$ was applied at some step $i$ in $D$, {\it i.e.,} $F_{i+1} = \alpha(\sigma_i(F_i), \fun{pos}(R), \pi_i)$. We say that this application is \emph{blocked} if one of the $\pi_i(B^-_q)$ (for any negative body $B_q^-$ in $R$) can be found in $A$. This can happen in two ways. Either $\pi_i(B^-_q)$ can already be found in $\sigma_i(F_i)$ or it appears later in the derivation.
In both  cases, there is a $\sigma_j(F_j)$ (with $j \geq i$) that contains the atomset $\pi_i(B^-_q)$, \emph{as transformed by the sequence of simplifications from $F_i$ to $F_j$}, {\it i.e.,} there exists $F_j$ with $j \geq i$ s.t. the atomset $\sigma_{i \rightarrow j}(\pi_i(B^-_q)) = \sigma_j(\ldots(\sigma_{i+1}(\pi_i(B^-_q)))\ldots)$ is included in $\sigma_j(F_j)$.
We say that a derivation $D$ is \emph{sound} when no rule application is blocked in $A$.  A sound derivation is said to be \emph{complete} when adding any other rule application to the derivation would either make it unsound, or would not change the produced atomset. The derivation is a $C$-chase when the $\sigma_i$ used at each step is determined by the criterion $C$.

\begin{definition}[$C$-stable sets]\label{def:cstable} Let $F$ be a finite atomset, and $\mathcal R$ be a set of NME rules. We say that a (possibly infinite) atomset $A$ is $C$-stable for $(F, \mathcal R)$ if there is a complete sound nonmonotonic $C$-chase from $F$ that produces $A$.
\end{definition}

\begin{proposition} If $\mathcal R$ is a set of existential rules, then there is a unique $C$-stable set, which is equivalent to the universal model $(F,\mathcal R)^C$. If $\{F\} \cup \mathcal R$ is a set of skolemized NME rules (with $F$ being seen as a rule with empty body), then its skolem-stable sets  are in bijection with its stable models.
\end{proposition}

\noindent\emph{Sketch of proof:} First part of the claim stems from the fact that existential rules generate a unique branch that corresponds to a derivation. When that branch is complete, it corresponds to a chase. Second part of the claim comes from the fact that our definitions mimic the behavior of the sound and complete algorithm implemented in \cite{asperix09}.$\hfill \Box$

\subsubsection{$C$-chase Tree} The problem with the fixpoint Definition~\ref{def:cstable} is that it does not provide an effective algorithm: at each step of the derivation, we need to know the set produced by that derivation. The algorithm used in the solver ASP\'eRIX  \cite{asperix09} 
is here generalized to a procedure that generates the (possibly infinite) \emph{$C$-derivation tree} of $(F, \mathcal R)$. All nodes of that tree are labeled by three fields. The field {\sc in} contains the atomset that was inferred in the current branch. The field {\sc out} contains the set of forbidden atomsets, {\it i.e.,} that must not be inferred. Finally, the field {\sc mbt} (``must be true'') contains the atomset that has yet to be proven. A node is called \emph{unsound} when a forbidden atomset has been inferred, or has to be proven, {\it i.e.,} when $\fun{\sc out} \cap (\fun{\sc in} \cup \fun{\sc mbt}) \not= \emptyset$. At the initial step, the root of the $C$-derivation tree is a positive node labeled $(\sigma_0(F), \emptyset, \emptyset)$. Then, let us chose a node $N$ that is not unsound and  has no child. Assume there is a rule $R = B^+, \fun{\bf not}B_1^-, \ldots, \fun{\bf not}B_k^- \rightarrow H$ in $\mathcal R$ such that there is a homomorphism $\pi$ from $B^+$ to $\fun{\sc in}(N)$. Then we will (possibly) add $k+1$ children under $N$, namely $N^+, N_1^-, \ldots, N_k^-$. These children are added if the rule application is not blocked, and produces new atoms. Intuitively, the positive child $N^+$ encodes the effective application of the rule, while the $k$ negative children $N_i^-$ encode the $k$ different possibilities of blocking the rule (with each of the negative bodies). Let us consider the sequence of positive nodes from the root of the tree to $N^+$. It encodes a $\fun{pos}(\mathcal R)$-derivation from $F$. On that derivation, the $C$-chase  generates a sequence $\sigma_0(F), \ldots, \sigma_{p}(F_{p}), S = \sigma(\alpha(\sigma_{p}(F_{p}), \fun{pos}(R), \pi))$. $S$ produces something new when $S \not\subseteq \sigma_{p}(F_{p})$.
We now have to fill the fields of the obtained children: let ({\sc in}, {\sc out}, {\sc mbt}) be the label of a node $N$. Then $\fun{label}(N^+) = (S, \fun{\sc out} \cup \{\pi_i(B^-_1), \ldots, \pi_i(B_k^-)\}, \fun{\sc mbt})$ and
$\fun{label}(N^-_i) = (\fun{\sc in}, \fun{\sc out}, \fun{\sc mbt} \cup \pi_i(B^-_i))$.

We say that a (possibly infinite) branch in the $C$-derivation tree is \emph{unsound} when it contains an unsound node. A sound branch is said to be \emph{complete} when its associated  derivation is complete. Finally, a sound and complete branch is \emph{stable} when for every node $N$ in the branch such that $B^- \in \fun{\sc mbt}(N)$, there exists a descendant $N'$ of $N$ such that $B^- \in \fun{\sc in}(N')$. We say that a branch is \emph{unprovable} if there exists a node $N$ in the branch and an atomset $B^- \in \fun{\sc mbt}(N)$ such that no complete branch containing $N$ is stable. We call a \emph{$C$-chase tree} any $C$-derivation tree for which all branches are either unsound, unprovable or complete.


\begin{proposition} An atomset $A$ is a $C$-stable set for $(F, \mathcal R)$ iff a $C$-chase tree of $(F, \mathcal R)$ contains a stable branch whose associated derivation produces $A$.
\end{proposition}

%



%
%
%
%
%

\subsubsection{On the applicability of the chase variants}

In the positive case, all chase variants produce equivalent universal models (up to skolemization). Moreover, running a chase on equivalent  knowledge bases produce equivalent results.  Do these semantic properties still hold with nonmonotonic existential rules? The answer is no in general.

The next example shows that the chase variants presented in this paper, core chase excepted, may produce non-equivalent results from equivalent knowledge bases.

\begin{example} Let $F = \{p(a, y), t(y)\}$ and $F' = \{p(a, y'), p(a, y), t(y)\}$ be two equivalent atomsets. Let $R : p(u, v), \fun{\bf not} \, t(v) \rightarrow r(u)$. For any $C$-chase other than core chase, there is a single $C$-stable set for $(F, \{R\})$ which is $F$ (or \emph{sk(F)}) and a single $C$-stable set for $(F', \{R\})$ which is $F' \cup \{r(a)\}$ (or $sk(F') \cup \{r(a)\}$). These sets are not equivalent.
\end{example}

Of course, if we consider that the initial knowledge base is already skolemized (including $F$ seen as a rule), this trouble does not occur with the skolem-chase since there are no redundancies in facts and no redundancy can be created by a rule application. This problem does not arise with core chase either. Thus the only two candidates for processing NME rules are the core chase and the skolem chase (if we assume \emph{a priori }skolemisation, which is already a semantic shift).



The choice between both mechanisms is important since, as shown by the next example, they may produce different results even when they both produce a \emph{unique} $C$-stable set.  It follows that skolemizing existential rules is not an innocuous transformation in presence of nonmontonic negation.

\begin{example} We consider $F = i(a)$, $R_1 = i(x) \rightarrow p(x,y)$, $R_2 = i(x) \rightarrow q(x,y)$, $R_3 = q(x,y) \rightarrow p(x,y), t(y)$ and $R_4 = p(u,v), \fun{\bf not} \, t(v) \rightarrow r(u)$.
The core chase produces at first step $p(a,y_0)$ and $q(a,y_1)$, then  $p(a,y_1)$ and $t(y_1)$ and removes the redundant atom $p(a,y_0)$, hence $R_4$ is not applicable. The unique core-stable set is $\{i(a), q(a,y_1), p(a,y_1), t(y_1) \}$. With the skolem chase, the produced atoms are $p(a,f^{R_1}(a))$ and $q(a,f^{R_2}(a))$, then $p(a,f^{R_2}(a))$ and $t(f^{R_2}(a))$. $R_4$ is applied with $p(u,v)$ mapped to $p(a,f^{R_1}(a))$, which produces $r(a)$. These atoms yield a unique skolem-stable set. These stable sets are not equivalent.
\end{example}

%

\section{Termination of the Chase Tree}

\subsubsection{On the finiteness of $C$-chase trees}

We say that the \emph{$C$-chase-tree halts} on $(F, \mathcal R)$ when there exists a finite $C$-chase tree of $(F, \mathcal R)$ (in that case, a breadth-first strategy for the rule applications will generate it). We can thus define \emph{$C$-stable-finite} as the class of sets of nonmonotonic existential rules $\mathcal R$ for which the $C$-chase-tree halts on any $(F, \mathcal R)$. Our first intuition was to assert ``if $\fun{pos}(\mathcal R) \in$ $C$-finite, then $\mathcal R \in$ $C$-stable-finite''. However, this property is not true in general, as shown by the following example:

\begin{example}
Let $\mathcal R = \{R_1, R_2\}$ where $R_1 = h(x) \rightarrow p(x, y), h(y)$ and $R_2 = p(x, y), \fun{\bf not} \, h(x) \rightarrow p(x, x)$. See that $\fun{pos}(\mathcal R) \in$ core-finite (as soon as $R_1$ is applied, $R_2$ is also applied and the loop $p(x, x)$ makes any other rule application redundant); however the only core-stable set of $(\{h(a)\}, \mathcal R)$ is infinite (because all applications of $R_2$ are blocked).
\end{example}

The following property  shows that the desired property is true for \emph{local} chases.

\begin{proposition} Let $\mathcal R$ be a set of NME rules and $C$ be a local chase. If $\fun{pos}(\mathcal R) \in$ $C$-finite, then $\mathcal R \in$ $C$-stable-finite.
\end{proposition}

We have previously argued that the only two interesting chase variants w.r.t. the desired semantic properties are skolem
and core. However, the core-finiteness of the positive part of a set of NME rules does not ensure the core-stable-finiteness of these rules. We should point out now that if $C \geq C'$, then $C'$-stable-finiteness implies $C$-stable-finiteness. We can thus ensure core-stable-finiteness when $C$-finiteness of the positive part of rules is ensured for a local $C$-chase.

\begin{proposition} Let $\mathcal R$ be a set of NME rules and $C$ be a local chase. If $\fun{pos}(\mathcal R) \in$ $C$-finite, then $\mathcal R \in$ core-stable-finite.
\end{proposition}

We can rely upon all acyclicity results in this paper to ensure that the core-chase tree halts.

\subsubsection{Improving finiteness results with negative bodies}


%
%
%
We now explain how negation can be exploited to enhance preceding acyclicity notions. We first define the notion of {\em self-blocking rule},
which is a rule that will never be applied in any derivation. A rule $B^+, \fun{\bf not}\, B^-_1, \ldots, \fun{\bf not} \, B_k^-$ is self-blocking if there is
a negative body $B^-_i$ such that $B^-_i \subseteq (B^+ \cup H)$. Such a rule will never be applied in a sound way, so will never
produce any atom. It follows that:

\begin{proposition}
Let $\mathcal R'$ be the non-self-blocking rules of $\mathcal R$. If $\fun{pos}(\mathcal R') \in$ $C$-finite and $C$ is local, then $\mathcal R \in$ $C$-stable-finite.
\end{proposition}

This idea can be further extended. We have seen for existential rules that if $R'$ depends on $R$, then there is a unifier $\mu$
of $\fun{body}(R')$ with $\fun{head}(R)$, and we can build a rule $R'' = R \unified_\unifier R'$ that captures the sequence of
applications encoded by the unifier.
We extend Def.~\ref{def:unified-rule} to take into account negative bodies: if $B^-$ is a negative body of $R$ or $R'$, then $\mu(B^-)$
is a negative body of $R''$. We also extend the notion of dependency in a natural way, and say that a unifier $\mu$ of $\fun{head}(R)$ with  $\fun{body}(R')$ is self-blocking
when $R \unified_\unifier R'$ is self-blocking, and $R'$ \emph{depends} on $R$ when there exists a unifier of $\fun{head}(R)$ with $\fun{body}(R')$ that is not self-blocking.
This extended notion of dependency exactly corresponds to the \emph{positive reliance} in \cite{magka2013computing}.

\begin{example} Let $R =  q(x), \fun{\bf not} \, p(x) \rightarrow r(x, y)$ and $R' = r(x, y) \rightarrow p(x), q(y)$. Their associated positive rules are not core-finite. There is a single unifier $\mu$ of $R'$ with $R$, and $R \unified_\unifier R': q(x), \fun{\bf not} \, p(x) \rightarrow r(x, y), p(x), q(y)$ is self-blocking. Then the skolem-chase-tree halts on $(F, \{R, R'\})$ for any $F$.
\end{example}

Results obtained from positive rules can thus be generalized by considering this extended notion of dependency
(for $\upg$ we only encode non self-blocking unifiers). Note that it does not change the complexity of the acyclicity tests.


%

%


We can further generalize this and check if a unifier sequence is self-blocking,
thus extend the $Y^{U+}$ classes to take into account negative bodies.
Let us consider a compatible cycle $C$ going through $\vpos{a}{i}$  that has not been proven safe. Let
$C_\mu$ be the set of all compatible unifier sequences induced by $C$.
 We say that a sequence $\mu_1 \ldots \mu_k \in \mathcal C_\mu$ is self-blocking when the rule
$R_1 \unified_{\mu_1} R_2 \ldots R_k \unified_{\mu_k} R_{k+1}$ obtained by combining these unifiers is self-blocking.
When all sequences in $C_\mu$ are self-blocking, we say that $C$ is also self-blocking. This test comes again at no additional computational cost.

\begin{example}
\label{ex:selfblock-unifier}
Let $R_1 = q(x_1), {\bf not} p(x_1) \rightarrow r(x_1,y_1)$, $R_2 =  r(x_2,y_2) \rightarrow s(x_2,y_2)$,
$R_3 = s(x_3,y_3) \rightarrow p(x_3), q(y_3)$.
$PG^{U+}(\{R_1,R_2, R_3\})$ has a unique cycle, with a unique induced  compatible unifier sequence. The rule
$R_1 \unified R_2 \unified R_3 = q(x_1), {\bf not} p(x_1) \rightarrow r(x_1,y_1), s(x_1,y_1), p(x_1), q(y_1)$
 is self-blocking, hence $R_1 \unified R_2 \unified R_3 \unified R_1$ also is. Thus,  there is no ``dangerous" cycle.
\end{example}

\begin{proposition}
\label{prop:neg-correct}
If, for each existential position $\vpos{a}{i}$, all compatible cycles for $\vpos{a}{i}$ in $PG^U$ are self-blocking, then
the stable computation based on the skolem chase halts.
\end{proposition}

\section{Conclusion}
We have revisited chase termination with several results. First, a new tool that allows to unify and extend most existing acyclicity conditions, while keeping good computational properties. Second, a chase-like mechanism for nonmonotonic existential rules under stable model semantics, as well the extension of acyclicity conditions to take negation into account. This latter contribution extends the notion of negative reliance of \cite{magka2013computing}; and does not rely upon stratification (and thus does not enforce the existence of a single stable model).

This work will be pursued on the theoretical side by a complexity study of {\sc entailment} for the new acyclic classes and  by a deeper study of logical foundations for NME rules, since it remains to relate our core-stable sets to an existing first-order semantics for general NME rules.

\section{ Acknowledgements}

We thank the reviewers for their comments. This work is part of the ASPIQ and Pagoda projects and was partly funded by the french \emph{Agence Nationale de la Recherche} (ANR) grants ANR-12-BS02-0003 and ANR-12-JS02-0007.


\bibliography{bib}
\bibliographystyle{aaai}

\end{document}